\documentclass[letterpaper, 10pt, conference]{ieeeconf}
\IEEEoverridecommandlockouts
\usepackage{amsmath}
\usepackage{amssymb}
\usepackage{graphicx}
\usepackage{booktabs}
\usepackage{multirow}
\usepackage{cite}
\usepackage{url}
\usepackage{xcolor}
\usepackage{placeins}
\usepackage{flafter}
\usepackage{dblfloatfix}
\usepackage{adjustbox}
\usepackage{flushend}
\usepackage[colorlinks=true, urlcolor=black]{hyperref}

\newtheorem{proposition}{Proposition}
\newtheorem{remark}{Remark}

\title{\LARGE \bf
WBAG: A Whole-Body and Attached-Geometry Safety Framework for Vision-Language-Action Manipulation}

\author{
Samuel Zhen$^{1}$,
Siwon Jo$^{2}$,
Yanze Zhang$^{3}$,
and Wenhao Luo$^{1}$%
\thanks{$^*$This work was supported in part by the U.S. National Science Foundation under Grant 2530297.}
\thanks{$^{1}$Department of Computer Science and Engineering,
Texas A\&M University, College Station, TX 77843, USA.
Email: {\tt \{szhen, wenhaol\}@tamu.edu}}
\thanks{$^{2}$GRASP Laboratory, University of Pennsylvania, Philadelphia, PA 19104, USA.
        Email: {\tt jo16@seas.upenn.edu}}
\thanks{$^{3}$Department of Computer Science, University of Illinois Chicago, Chicago, IL 60607, USA. Email:{\tt \{yzhan361\}@uic.edu}}
}

\begin{document}

\maketitle
\thispagestyle{empty}
\pagestyle{empty}

\begin{abstract}
Vision-language-action (VLA) policies have demonstrated impressive capabilities in generalizable robotic manipulation, but their deployment in the real world remains challenging due to potential collisions involving different parts of the robot, manipulated objects, and the surrounding environment. Existing inference-time VLA safety frameworks typically rely on simplified end-effector-centered representations that do not explicitly model the full articulated robot and attached-object geometry. In this paper, we present WBAG, a safety framework that models the robot's whole-body and grasp-dependent attached geometry. WBAG constructs a grasp-conditioned safe set that adapts the protected geometry as objects are grasped, then converts this evolving geometry into differentiable CBF constraints that minimally modify the VLA's six-dimensional operational-space action for collision avoidance across robot, scene, and attached geometry. On the SafeLIBERO benchmark, a variant of LIBERO augmented with obstacles for safety evaluation, WBAG achieves the best overall safety and safe task success among the evaluated methods under a scene-level safety evaluator that monitors all eligible non-task objects, reaching 97.38\% aggregate Scene Safety and 59.38\% Safe Success. \textbf{Project page:} \url{https://samuelzhen.com/projects/wbag}
\end{abstract}

\section{Introduction}

Vision-language-action (VLA) models map visual observations and language
instructions directly to robot actions across diverse manipulation tasks
\cite{pi05}, but their task-level generalization does not provide an explicit
geometric safety model to consistently safeguard robot motions. In cluttered environments, a VLA action may move the
end effector toward the target while causing an upstream link, wrist, gripper,
or manipulated object to collide. The relevant collision geometry spans the
articulated robot and changes when the robot grasps an object. Recent benchmark
evaluations further show that task completion alone can obscure physical safety
violations during manipulation~\cite{fan2026safevlabench,cui2026liberosafety}.

\begin{figure}[!t]
    \centering
    \includegraphics[width=0.90\columnwidth]{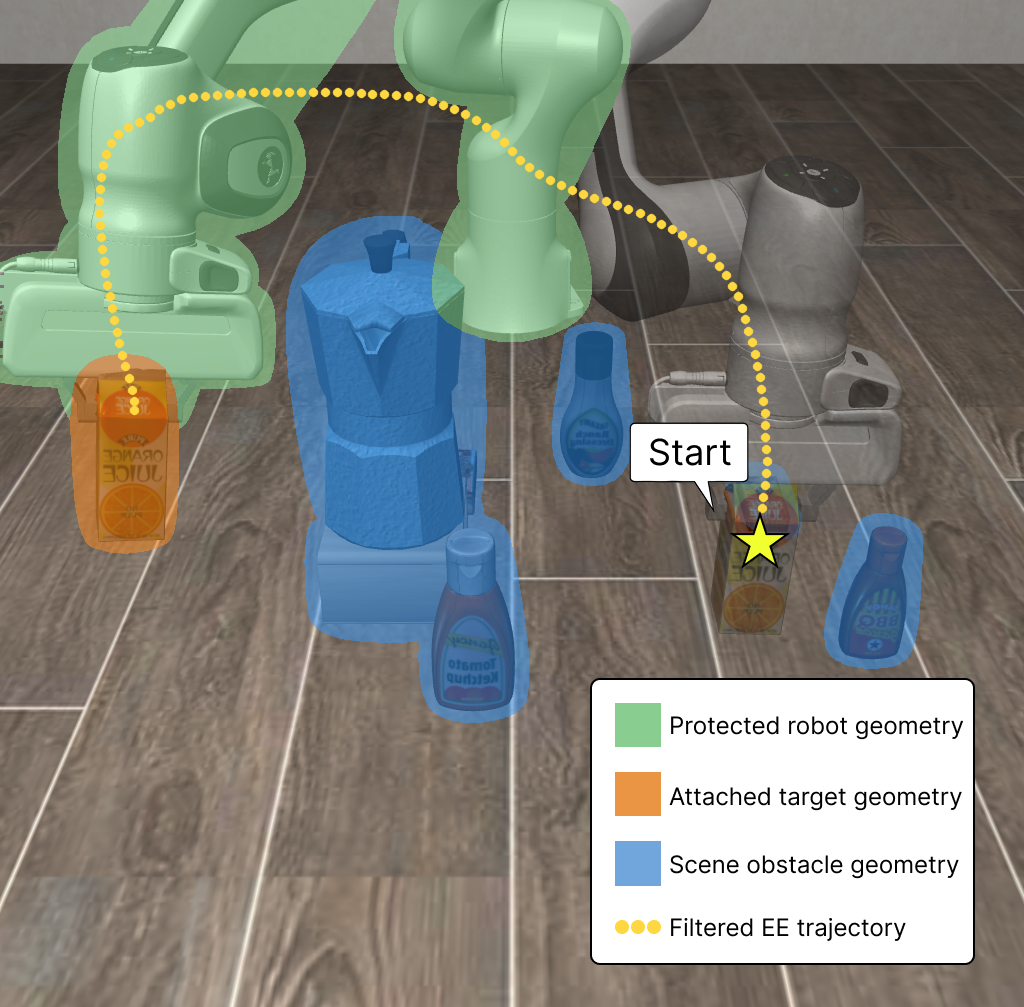}
    \caption{Example of our WBAG execution. Green shows the constructed protected whole-body robot
    geometry, and orange shows the attached target geometry after grasping, with blue
    showing the extracted scene obstacle geometry. The yellow dotted curve is the
    filtered safe end-effector trajectory.}
    \label{fig:qualitative_comparison}
    \vspace{-18pt}
\end{figure}

Control-theoretic safety methods, particularly control barrier functions (CBFs),
provide a principled mechanism for enforcing collision avoidance during robotic
control \cite{ames_cbf_qp}. CBFs constrain control inputs to maintain forward invariance of a prescribed
safe set, but their effectiveness depends on how that set is represented. For
geometric collision avoidance, the safety function must capture separation
between the protected robot geometry and surrounding obstacles. The geometric
representation therefore determines which collision modes can be enforced.

Recent work on safety-critical control with VLA-based policies \cite{safevla,vlsa_aegis,knows,english2026neurosymbolic,sinaei2026barrier} spans both training-time and inference-time interventions.
SafeVLA \cite{safevla} incorporates safety constraints during policy optimization through
constrained reinforcement learning. At inference time, AEGIS \cite{vlsa_aegis}
combines perception-derived geometric information with CBF constraints to filter
a pretrained VLA, while KNOWS \cite{knows} combines VLA-based target
identification with online obstacle tracking. Other methods enforce safety directly within flow-matching action generation.
Recent work in~\cite{english2026neurosymbolic} apply predictive
trajectory-level barrier constraints through minimum-norm corrections during
iterative denoising, while~\cite{sinaei2026barrier} modify the
flow-matching dynamics using an aggregate CBF over the generated action chunk.
However, existing
inference-time safety frameworks
typically enforce collision avoidance using simplified end-effector-centered
geometry \cite{vlsa_aegis,knows, english2026neurosymbolic}. Such end-effector-centered representations do
not explicitly model the full articulated robot or the rigid geometry of an
attached object: the end effector may remain
separated while an upstream link, wrist, or attached geometry collides with an
obstacle.

Geometric methods such as Bernstein-polynomial signed distance fields (BP-SDFs) \cite{bpsdf} provide smooth, differentiable representations of object geometry, enabling efficient distance queries and gradients that are well suited for optimization-based control. For articulated robots, BP-SDFs can represent individual links, whose poses are determined by forward kinematics as a function of robot configuration. Recent work \cite{geometry_cbf} employs BP-SDFs to construct geometry-aware CBFs from the minimum distance between robot and obstacle geometries and derive corresponding safety constraints on the robot control input. Applying such geometric safety constraints to VLA manipulation requires connecting perception, manipulation state, and the policy's native action interface. In our setting, scene and target geometry are reconstructed from RGB-D observations rather than assumed known, while the protected geometry changes when the target becomes rigidly attached after grasping. At the same time, the resulting link- and object-level clearances are functions of robot configuration, whereas the frozen VLA produces six-dimensional operational-space actions. These factors must therefore be coupled so that perception-derived, grasp-conditioned geometric constraints can be enforced directly on the VLA action without retraining the VLA or modifying the operational-space controller.

In this paper, we introduce \textbf{WBAG}, a whole-body and attached-geometry safety framework for pretrained VLA policies. WBAG reconstructs scene and target geometry from RGB-D observations and constructs a \textbf{grasp-conditioned safe set}: before grasping, it protects the articulated robot; after grasping, the protected set expands to include the rigidly attached object. We represent robot, scene, and attached geometry with differentiable BP-SDFs and enforce the resulting CBF constraints directly on the VLA's native six-dimensional operational-space action. WBAG minimally modifies the policy output while accounting for collisions involving the whole robot and attached geometry, without retraining the underlying VLA or modifying the operational-space controller. We evaluate WBAG on SafeLIBERO using a scene-level safety evaluator that monitors all eligible non-task scene objects. WBAG achieves the highest aggregate Scene Safety and Safe Success among the evaluated configurations, reaching \(97.38\%\) and \(59.38\%\), respectively.
Our main contributions are summarized as the following:
\begin{itemize}
    \item We formulate a \textbf{grasp-conditioned geometric CBF} over articulated robot links and dynamically attached geometry, with the protected set adapting across pre- and post-grasp manipulation.

    \item We develop \textbf{WBAG}, a novel perception-to-control safety framework that reconstructs scene and target geometry from visual observations, maintains grasp-conditioned protected geometry, and enforces the resulting constraints directly on the VLA’s native operational-space action without retraining the VLA or modifying the operational-space controller used to execute its actions.

    \item We conduct a \textbf{controlled SafeLIBERO evaluation} with scene-level safety measurement that isolates the effects of scene coverage, whole-body protection, and attached-geometry modeling on safe VLA manipulation.
\end{itemize}

\begin{figure*}[!t]
    \centering
    \includegraphics[width=\textwidth]{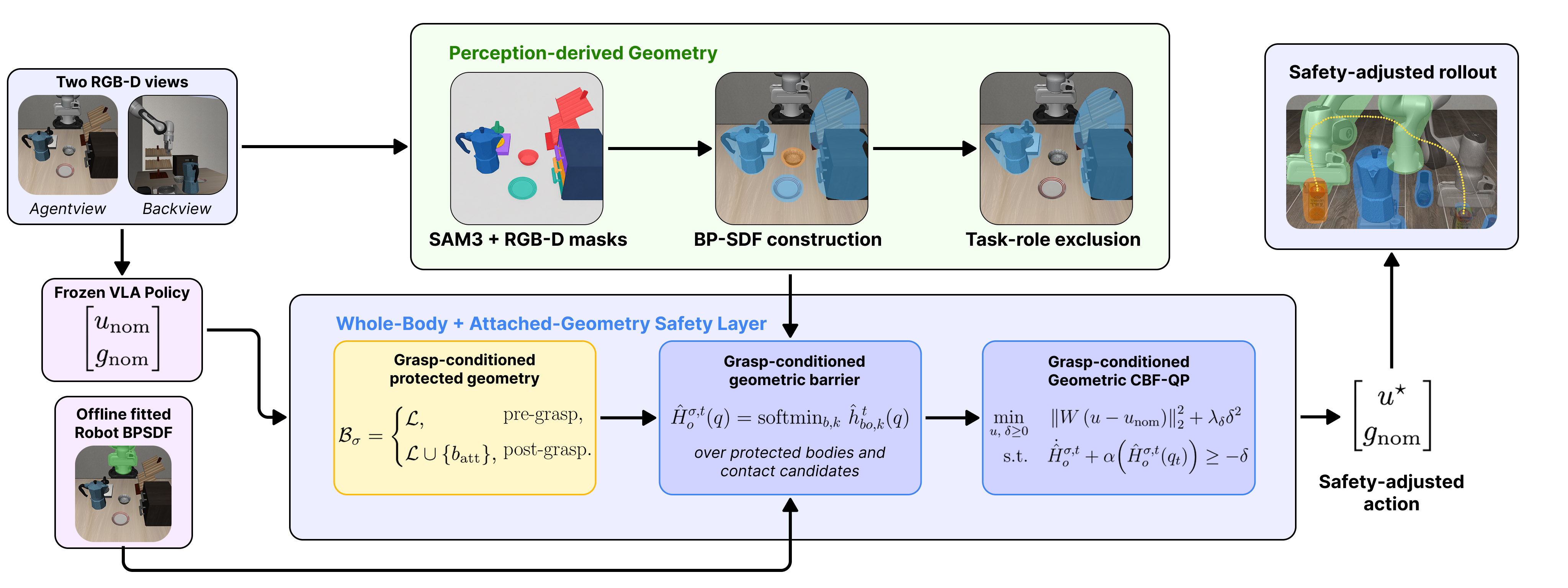}
    \caption{Overview of WBAG. RGB-D observations reconstruct scene geometry,
    while robot BP-SDFs are fitted offline. Task-role exclusion defines scene
    obstacles, and the protected set expands from the articulated robot to
    include attached geometry after grasp confirmation. The resulting composite
    clearances are enforced by the grasp-conditioned geometric CBF-QP.
    }
    \label{fig:wbag_framework}
\end{figure*}

\section{Preliminaries}

\subsection{Control Barrier Functions}

Consider the control-affine system
\begin{equation}
    \dot{x} = f(x) + g(x)u,
    \label{eq:control_affine}
\end{equation}
where $x\in\mathcal{X}\subseteq\mathbb{R}^{n_x}$ is the system state,
$u\in\mathcal{U}\subseteq\mathbb{R}^{n_u}$ is the control input, and
$f:\mathcal{X}\rightarrow\mathbb{R}^{n_x}$ and
$g:\mathcal{X}\rightarrow\mathbb{R}^{n_x\times n_u}$ are locally
Lipschitz continuous.
Let $h:\mathcal{X}\rightarrow\mathbb{R}$ be continuously differentiable and
define the safe set
\begin{equation}
    \mathcal{C}
    =
    \{x\in\mathcal{X}:h(x)\geq0\}.
    \label{eq:generic_safe_set}
\end{equation}
Following the standard CBF formulation~\cite{ames_cbf_qp}, $h$ is a
control barrier function if there exists an extended class-$\mathcal{K}$
function $\alpha$ such that
\begin{equation}
    \sup_{u\in\mathcal{U}}
    \left\{
    \nabla_x h(x)^\top
    \left(f(x)+g(x)u\right)
    \right\}
    \geq
    -\alpha\!\left(h(x)\right),
    \label{eq:cbf_definition}
\end{equation}
for all $x\in\mathcal{X}$. Accordingly, the admissible control set is
\begin{equation}
    \mathcal{B}(x)
    =
    \left\{
    u\in\mathcal{U}:
    \nabla_x h(x)^\top
    \left(f(x)+g(x)u\right)
    +
    \alpha\!\left(h(x)\right)
    \geq0
    \right\}.
    \label{eq:cbf_admissible_set}
\end{equation}
Any Lipschitz continuous controller satisfying
$u(x)\in\mathcal{B}(x)$ renders $\mathcal{C}$ forward invariant.

For geometric collision avoidance, $h$ can be defined from robot--obstacle
clearance.

\subsection{Bernstein Polynomial Signed Distance Fields}

We represent rigid-body geometry using Bernstein polynomial signed distance
fields (BP-SDFs)~\cite{bpsdf,geometry_cbf}. For normalized local coordinates
$\xi\in[0,1]^3$, let $B(\xi)\in\mathbb{R}^{Q^3}$ denote the
tensor-product Bernstein basis of order $Q$. The local-frame BP-SDF of a
rigid body $b$ is the scalar-valued function
$\bar{\phi}_b:[0,1]^3\rightarrow\mathbb{R}$,
\begin{equation}
    \bar{\phi}_b(\xi) = c_b^\top B(\xi),
    \label{eq:bpsdf}
\end{equation}
where $c_b\in\mathbb{R}^{Q^3}$ contains the fitted Bernstein coefficients.
The value $\bar{\phi}_b(\xi)\in\mathbb{R}$ approximates the signed distance
to the surface of body $b$ at the spatial location represented by $\xi$.
We use the convention that $\bar{\phi}_b>0$ outside the body,
$\bar{\phi}_b=0$ on its surface, and $\bar{\phi}_b<0$ inside the body.

For an articulated robot, let $q\in\mathbb{R}^{n_q}$ denote the robot
configuration. Forward kinematics determines the pose of each robot link
$b$ as a function of $q$. For a world-frame query point
$x\in\mathbb{R}^3$, let $\xi_b(x;q)\in[0,1]^3$ denote its normalized local
coordinates in the frame of link $b$. The corresponding world-frame BP-SDF is
\begin{equation}
    \phi_b(x;q) = \bar{\phi}_b\!\left(\xi_b(x;q)\right),
    \label{eq:world_bpsdf}
\end{equation}
where $\phi_b:\mathbb{R}^3\times\mathbb{R}^{n_q}\rightarrow\mathbb{R}$
returns the signed distance from $x$ to the represented geometry of link $b$
at configuration $q$.

The smooth Bernstein representation provides spatial distance gradients, and
its composition with forward kinematics makes $\phi_b(x;q)$ differentiable
with respect to $q$, enabling geometric clearance to be incorporated into
optimization-based control and the geometric CBF construction discussed in Section~III. 
\subsection{Problem Formulation}

Let $\pi$ denote a frozen VLA policy, instantiated in our experiments with
$\pi_{0.5}$~\cite{pi05}, that maps an observation $o$ and language
instruction $\ell$ to a nominal action
\begin{equation}
    a_{\mathrm{nom}}
    =
    \pi(o,\ell)
    =
    \begin{bmatrix}
        u_{\mathrm{nom}} \\
        g_{\mathrm{nom}}
    \end{bmatrix}
    \in \mathbb{R}^{7},
    \label{eq:vla_action}
\end{equation}
where
\begin{equation}
    u_{\mathrm{nom}}
    =
    \begin{bmatrix}
        \Delta p \\
        \Delta r
    \end{bmatrix}
    \in \mathbb{R}^{6}
    \label{eq:vla_motion_action}
\end{equation}
is the operational-space motion command. Here,
$\Delta p\in\mathbb{R}^{3}$ and $\Delta r\in\mathbb{R}^{3}$ denote the
translational and axis--angle rotational increments, respectively, and
$g_{\mathrm{nom}}\in\mathbb{R}$ is the scalar gripper command.

Let $q\in\mathbb{R}^{n_q}$ denote the robot joint configuration and
$\mathcal{L}$ the set of modeled moving robot bodies. Let
$\sigma\in\{0,1\}$ denote the grasp mode, with $\sigma=0$ before grasp
confirmation and $\sigma=1$ after grasp confirmation. Let
$b_{\mathrm{att}}$ denote the attached target body, and define the
grasp-conditioned protected set
\begin{equation}
    \mathcal{B}_\sigma
    =
    \begin{cases}
        \mathcal{L}, &
        \sigma=0, \\[1mm]
        \mathcal{L}\cup\{b_{\mathrm{att}}\}, &
        \sigma=1.
    \end{cases}
    \label{eq:protected_set}
\end{equation}
Thus, the set of protected bodies changes from the articulated robot to the
robot--object system after attachment.

Let $\mathcal{O}$ denote the set of reconstructed scene bodies subject to collision-avoidance constraints. Objects requiring intentional task contact are excluded from $\mathcal{O}$, and the attached target is not constrained against itself. The protected bodies in $\mathcal{B}_\sigma$ and scene obstacles in $\mathcal{O}$ are represented using BP-SDFs.

Given $(u_{\mathrm{nom}},q,\sigma,\mathcal{O})$, our \textbf{objective} is to compute
a step-wise safety-adjusted command $u^\star\in\mathbb{R}^{6}$ that minimally
deviates from $u_{\mathrm{nom}}$ while maintaining collision-free operation
between every protected body $b\in\mathcal{B}_\sigma$ and every scene obstacle
$o\in\mathcal{O}$. The executed action is
\begin{equation}
    a^\star
    =
    \begin{bmatrix}
        u^\star \\
        g_{\mathrm{nom}}
    \end{bmatrix}
    \in\mathbb{R}^{7}.
    \label{eq:safe_vla_action}
\end{equation}
The safety filter therefore modifies only the operational-space motion command
while preserving the VLA's gripper command.

\section{Grasp-Conditioned Geometric Barrier}

We first derive the barrier using ideal closest-point evaluation. Section~IV
then introduces the finite-contact approximation used online.
\subsection{Pairwise BP-SDF Clearance}

For protected body $b\in\mathcal{B}_\sigma$, let
$\phi_b(x;q)$ denote its world-frame BP-SDF from
\eqref{eq:world_bpsdf}. For obstacle $o\in\mathcal{O}$, let
$\phi_o(x)$ denote its BP-SDF and define the offset obstacle level set
\begin{equation}
    \Gamma_o
    =
    \{x\in\mathbb{R}^3:\phi_o(x)=\ell_o\},
    \label{eq:obstacle_level_set}
\end{equation}
where $\ell_o>0$ is the obstacle offset.

For a given robot joint configuration $q$,
define
\begin{equation}
    x^\star_{bo}(q)
    =
    \arg\min_{x\in\Gamma_o}
    \phi_b(x;q).
    \label{eq:closest_point}
\end{equation}
Here $x^\star_{bo}(q)$ is the point on the obstacle offset surface $\Gamma_o$ with minimum BP-SDF distance to protected body $b$.
The corresponding pairwise clearance barrier is
\begin{equation}
    h_{bo}(q)
    =
    \phi_b\!\left(x^\star_{bo}(q);q\right)-\ell_b
    =
    \min_{x\in\Gamma_o}\phi_b(x;q)-\ell_b,
    \label{eq:pairwise_barrier}
\end{equation}
where $\ell_b\geq0$ is the protected-body margin. Thus, $h_{bo}(q)\geq0$ enforces the prescribed geometric separation between protected body $b$ and obstacle $o$.

Following the closest-point gradient result in ~\cite{geometry_cbf}, assume that the relevant protected-body and obstacle BP-SDF level sets are convex, that $x^\star_{bo}(q)$ varies smoothly with $q$, and that the corresponding BP-SDF gradients are nonzero on the level-set boundaries. The pairwise clearance gradient is then
\begin{equation}
    \nabla_q h_{bo}(q)
    =
    \nabla_q
    \phi_b\!\left(x^\star_{bo}(q);q\right).
    \label{eq:pairwise_barrier_gradient}
\end{equation}
\subsection{Grasp-Conditioned Composite Barrier}
The pairwise barrier $h_{bo}$ measures clearance for one protected-body--obstacle
pair. To enforce clearance for all protected bodies with a single constraint per
obstacle, we aggregate the pairwise barriers using a soft minimum:
\begin{equation}
    H_o^\sigma(q)
    =
    -\frac{1}{\beta}
    \log
    \left(
        \sum_{b\in\mathcal{B}_\sigma}
        e^{-\beta h_{bo}(q)}
    \right),
    \qquad \beta>0.
    \label{eq:composite_barrier}
\end{equation}
where $\beta$ is the soft-min sharpness parameter, with larger values
providing a tighter approximation to the minimum. Since $H_o^\sigma(q)\leq\min_{b\in\mathcal{B}_\sigma}h_{bo}(q)$, the
condition $H_o^\sigma(q)\geq0$ implies
$h_{bo}(q)\geq0$ for every $b\in\mathcal{B}_\sigma$.

We therefore define the grasp-conditioned safe set
\begin{equation}
    \mathcal{C}_\sigma
    =
    \bigcap_{o\in\mathcal{O}}
    \left\{
        q:
        H_o^\sigma(q)\geq0
    \right\}.
    \label{eq:grasp_safe_set}
\end{equation}

\subsection{Action-to-Motion Mapping and Fixed-Mode Safety}

For a fixed grasp mode $\sigma$, let $u\in\mathbb{R}^6$ denote an
operational-space control command. We locally relate $u$ to joint motion by
\begin{equation}
    \dot q = G(q)u,
    \label{eq:local_action_map}
\end{equation}
where $G(q)\in\mathbb{R}^{n_q\times6}$ is the local action-to-motion mapping.

To derive the corresponding control constraint, differentiating the soft-min barrier in~\eqref{eq:composite_barrier} gives
\begin{equation}
\begin{aligned}
    \nabla_q H_o^\sigma(q)
    &=
    \sum_{b\in\mathcal{B}_\sigma}
    w_{bo}(q)\nabla_q h_{bo}(q),\\[-1pt]
    w_{bo}(q)
    &=
    \frac{e^{-\beta h_{bo}(q)}}
    {\sum_{b'\in\mathcal{B}_\sigma}
    e^{-\beta h_{b'o}(q)}}.
\end{aligned}
\label{eq:softmin_gradient}
\end{equation}
Using~\eqref{eq:local_action_map}, its time derivative is therefore
\begin{equation}
    \dot H_o^\sigma(q)
    =
    \nabla_q H_o^\sigma(q)^\top G(q)u.
    \label{eq:composite_barrier_derivative}
\end{equation}
The corresponding CBF constraint thus becomes
\begin{equation}
    \nabla_q H_o^\sigma(q)^\top G(q)u
    +
    \alpha\!\left(H_o^\sigma(q)\right)
    \geq0,
    \qquad
    \forall o\in\mathcal{O}.
    \label{eq:fixed_mode_cbf}
\end{equation}

\begin{proposition}[Fixed-mode grasp-conditioned safety]
\label{prop:fixed_mode_safety}
Fix $\sigma\in\{0,1\}$ and suppose the convexity and regularity assumptions of \eqref{eq:pairwise_barrier_gradient} hold. If
$q(0)\in\mathcal{C}_\sigma$ and
\eqref{eq:fixed_mode_cbf} is satisfied for every
$o\in\mathcal{O}$, then $\mathcal{C}_\sigma$ is forward invariant under
\eqref{eq:local_action_map}.
\end{proposition}

\begin{proof}
For fixed $\sigma$, the protected set $\mathcal{B}_\sigma$ is fixed, and
each $h_{bo}$ is continuously differentiable under the convexity and regularity assumptions of \eqref{eq:pairwise_barrier_gradient}. Hence, the soft-min composition
$H_o^\sigma$ is continuously differentiable. Equation~\eqref{eq:fixed_mode_cbf}
is therefore the standard CBF condition for the set
$\{q:H_o^\sigma(q)\geq0\}$. Enforcing it for every
$o\in\mathcal{O}$ renders their intersection
$\mathcal{C}_\sigma$ forward invariant. Since
$H_o^\sigma\leq\min_{b\in\mathcal{B}_\sigma}h_{bo}$,
membership in $\mathcal{C}_\sigma$ also implies nonnegative clearance for
every represented protected-body--obstacle pair.
\end{proof}

\begin{remark}[Grasp-mode transition]
Proposition~\ref{prop:fixed_mode_safety} applies only within a fixed grasp mode.
At attachment, the protected set expands from $\mathcal{B}_0$ to
$\mathcal{B}_1$ by including the grasped object. The corresponding safe set
therefore introduces additional clearance constraints involving the attached
object. Although the robot configuration is continuous through the grasp
transition, a configuration in $\mathcal{C}_0$ may not belong to
$\mathcal{C}_1$ if the newly attached object is already too close to another
scene object. Thus, the proposition does not provide a safety guarantee across
the discrete grasp transition itself. If the robot configuration at attachment lies in
$\mathcal{C}_1$, the fixed mode forward invariance guarantee applies after
attachment. 
\end{remark}

\section{WBAG Safe Control Framework}

Section~III defines a grasp-conditioned geometric barrier that assumes exact continuous closest-point evaluation between each protected body and obstacle. 

WBAG instantiates this formulation as the
perception-to-control safety pipeline shown in Fig.~\ref{fig:wbag_framework}.
At episode initialization, multi-view RGB-D observations are used to
reconstruct scene and target geometry, while robot-link BP-SDFs are
precomputed offline. During execution, the protected set is updated when a
grasp is confirmed and the attached-object geometry is propagated with the
robot hand. The ideal closest-point barriers are approximated online using
finite local contact search, and the resulting geometric constraints are
mapped through a local action-to-motion model to the VLA's six-dimensional
operational-space command. WBAG then solves a quadratic program (QP) that minimally modifies the nominal VLA action while satisfying the grasp-conditioned geometric CBF constraints.

\subsection{Perception-derived Scene and Target Geometry}

To instantiate the geometric barriers in Section~III, WBAG combines precomputed robot geometry with perception-derived scene and target geometry. Robot-link BP-SDFs are precomputed offline from collision meshes using axis-wise second-difference constraints on the Bernstein coefficients to
promote convex BP-SDF level sets. At episode initialization, scene and target geometry are reconstructed from multi-view RGB-D observations, as shown in Fig.~\ref{fig:wbag_framework}. SAM3~\cite{sam3} provides object masks for the camera observations, and the corresponding masked depth observations are fused into world-frame point clouds. Task-role information identifies the target, goal, and support objects and determines which reconstructed components are excluded from the obstacle set.

For each reconstructed object, including the task target, we construct a convex hull of the observed points and fit a BP-SDF using the same convexity-promoting constraints. The retained non-task objects define the scene obstacle set $\mathcal{O}$, while the target BP-SDF is retained for incorporation into the protected set after grasping.

\begin{figure}[!t]
    \centering
    \includegraphics[width=0.50\columnwidth]{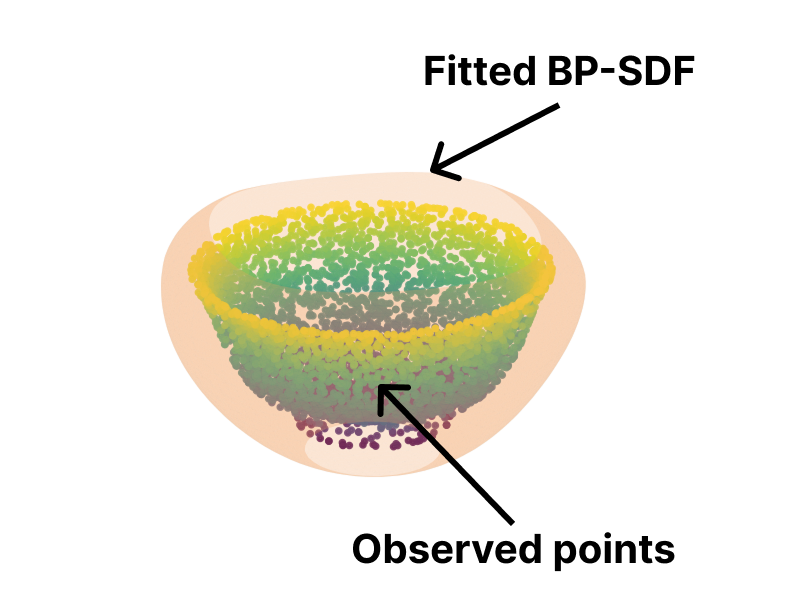}\\[-8pt]
    {\small (a) Target BP-SDF construction}
    \includegraphics[width=\columnwidth]{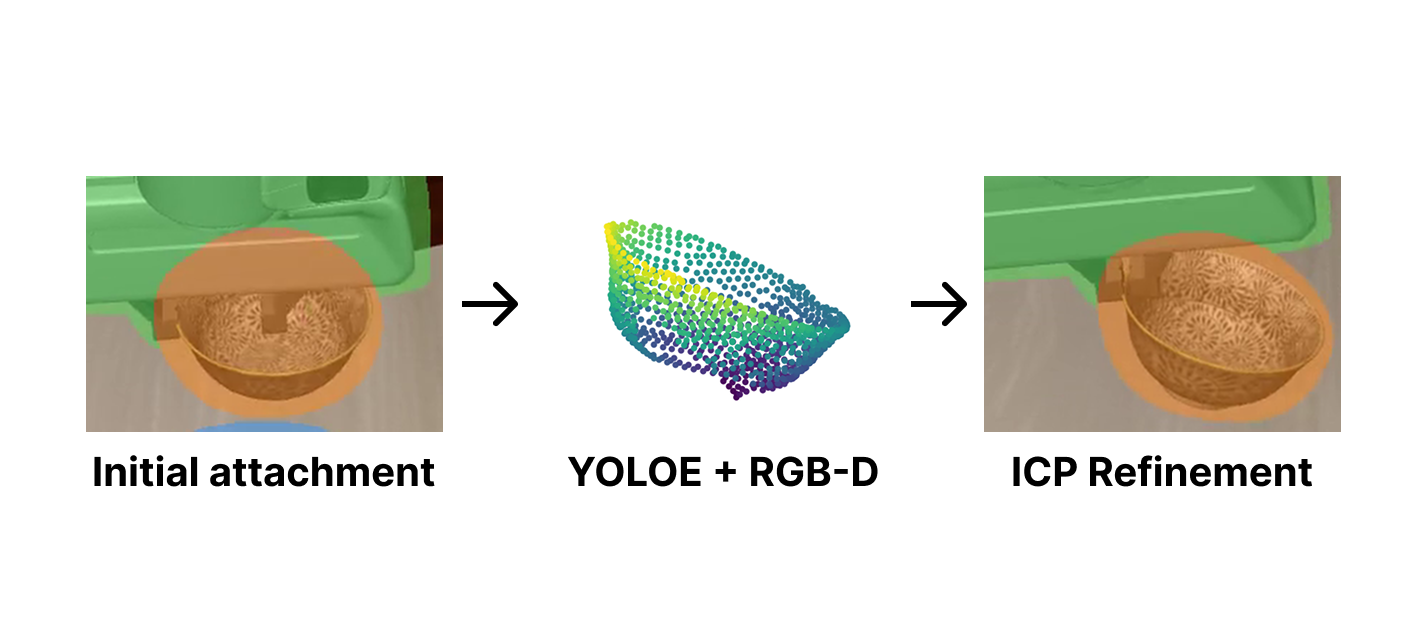}\\[-24pt]
    {\small (b) Post-grasp pose refinement}
    \caption{Attached-geometry construction and refinement. (a) The task target
    is reconstructed from RGB-D observations and represented by a fitted BP-SDF.
    (b) After grasp confirmation, target RGB-D points are registered to the
    initial model with ICP, refining the hand-relative attachment pose before
    rigid propagation with the gripper.
    }
    \label{fig:attached_geometry}
\end{figure}

\subsection{Grasp-Conditioned Attached Geometry}
To realize the grasp-conditioned protected set $\mathcal{B}_\sigma$ from Section~III, WBAG detects attachment and estimates the target pose relative to the robot hand. Grasp confirmation is determined from gripper closure and target proximity over consecutive control steps. Upon confirmation, we set $\sigma=1$; opening the gripper resets $\sigma=0$. Let $W$, $H$, and $A$ denote the world, hand, and attached-object model frames, respectively, and let $T_{XY}\in SE(3)$ denote the rigid transform from a frame $Y$ to another frame $X$. At grasp time $t_g$, the hand-relative attachment transform is initialized as
\begin{equation}
    T_{HA}^{t_g}
    =
    T_{WH}(q_{t_g})^{-1}
    \hat T_{WA}^{t_g},
    \label{eq:initial_attachment_transform}
\end{equation}
where $\hat T_{WA}^{t_g}$ is the perception-derived target pose obtained from the reconstruction in Section~IV-A. The object's pose relative to the hand can change immediately after grasping as it settles within the gripper, since contact with the gripper and the object's geometry may induce a post-grasp reorientation. WBAG therefore refines the attachment transform during a short post-grasp window. YOLOE-26~\cite{yoloe} is used for repeated target segmentation during this window because it provides computationally efficient segmentation suitable for online inference, and the corresponding RGB-D points are registered to the initial target model using ICP~\cite{icp} to update $T_{HA}$. After refinement, $T_{HA}$ is held fixed and the attached geometry is propagated rigidly with the hand according to
\begin{equation}
    T_{WA}(q_t)
    =
    T_{WH}(q_t)T_{HA}.
    \label{eq:attached_pose_propagation}
\end{equation}
This propagated pose determines the world-frame BP-SDF of the attached object. Analogous to the robot-link representation in~\eqref{eq:world_bpsdf}, we define
\begin{equation}
    \phi_{\mathrm{att}}(x;q)
    =
    \bar{\phi}_{\mathrm{att}}
    \!\left(\xi_{\mathrm{att}}(x;q)\right),
    \label{eq:attached_world_bpsdf}
\end{equation}
where $\xi_{\mathrm{att}}(x;q)$ denotes the coordinates of the world-frame query point $x$ expressed in the attached-object frame induced by $T_{WA}(q)=T_{WH}(q)T_{HA}$. Once grasped, the target becomes the protected body $b_{\mathrm{att}}\in\mathcal{B}_1$ as defined in \eqref{eq:protected_set}, and its BP-SDF enters the pairwise and composite barriers of Section~III in the same manner as the robot-link BP-SDFs.
\begin{table*}[!t]
    \centering
    \caption{Experimental progression from EEF-MVEE references to scene-aware
    EEF protection and WBAG.}
    \label{tab:method_configurations}
    \small
    \setlength{\tabcolsep}{5pt}
    \renewcommand{\arraystretch}{1.15}
    \begin{tabular*}{\textwidth}{@{\extracolsep{\fill}}lcccc@{}}
        \toprule
        Method & Perception & \shortstack{Obstacle\\geometry} &
        \shortstack{Protected robot\\geometry} & \shortstack{Attached geometry\\representation} \\
        \midrule
        AEGIS baseline~\cite{vlsa_aegis}
        & Grounding DINO~\cite{grounding_dino} & single MVEE & EEF ellipsoid & {\scriptsize handcrafted EEF proxy} \\
        SAM3 EEF-MVEE
        & SAM3~\cite{sam3} & single MVEE & EEF ellipsoid & {\scriptsize handcrafted EEF proxy} \\
        SAM3 Scene EEF-MVEE
        & SAM3~\cite{sam3} & scene MVEEs & EEF ellipsoid & {\scriptsize handcrafted EEF proxy} \\
        WBAG EEF-only + attached geometry
        & SAM3~\cite{sam3} & scene BP-SDFs & EEF-only & {\scriptsize attached BP-SDF} \\
        WBAG w/o attached geometry
        & SAM3~\cite{sam3} & scene BP-SDFs & whole-body BP-SDF & none \\
        \textbf{WBAG}
        & SAM3~\cite{sam3} & scene BP-SDFs & whole-body BP-SDF & {\scriptsize attached BP-SDF} \\
        \bottomrule
    \end{tabular*}
\end{table*}

\begin{figure*}[!t]
    \centering
    \begin{minipage}[t]{0.20\textwidth}
        \centering
        \includegraphics[width=\linewidth]{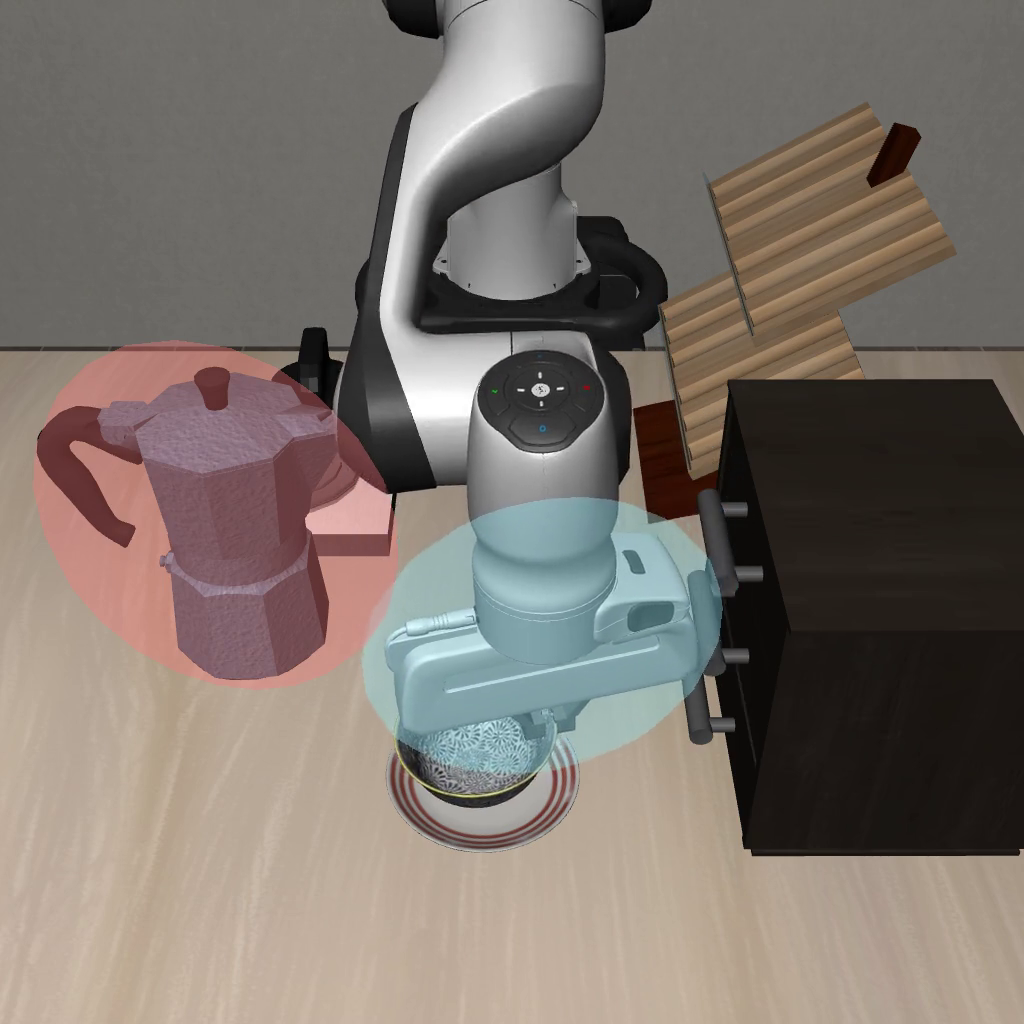}\\[-1pt]
        {\scriptsize\textbf{(a)} EEF-MVEE\\[-1pt] (AEGIS/SAM3)}
    \end{minipage}\hfill
    \begin{minipage}[t]{0.20\textwidth}
        \centering
        \includegraphics[width=\linewidth]{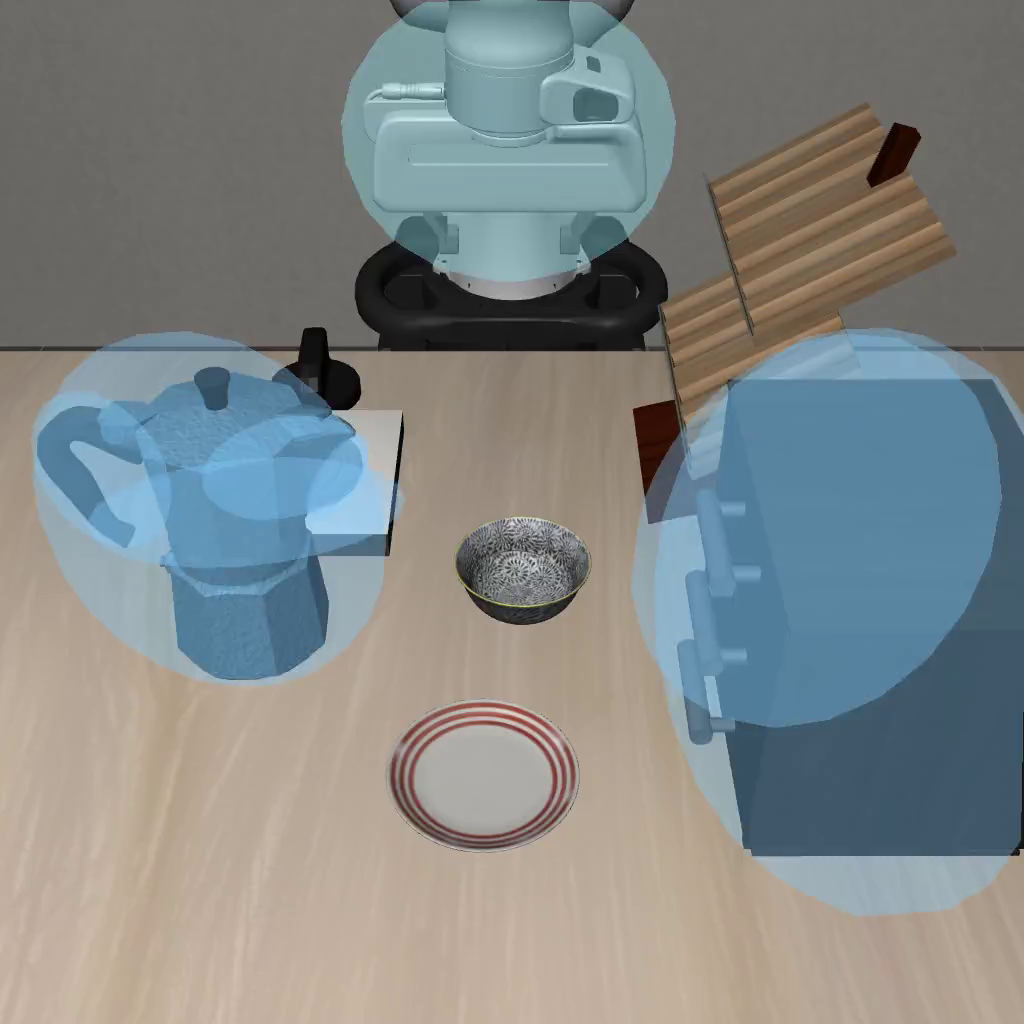}\\[-1pt]
        {\scriptsize\textbf{(b)} Scene EEF-MVEE}
    \end{minipage}\hfill
    \begin{minipage}[t]{0.20\textwidth}
        \centering
        \includegraphics[width=\linewidth]{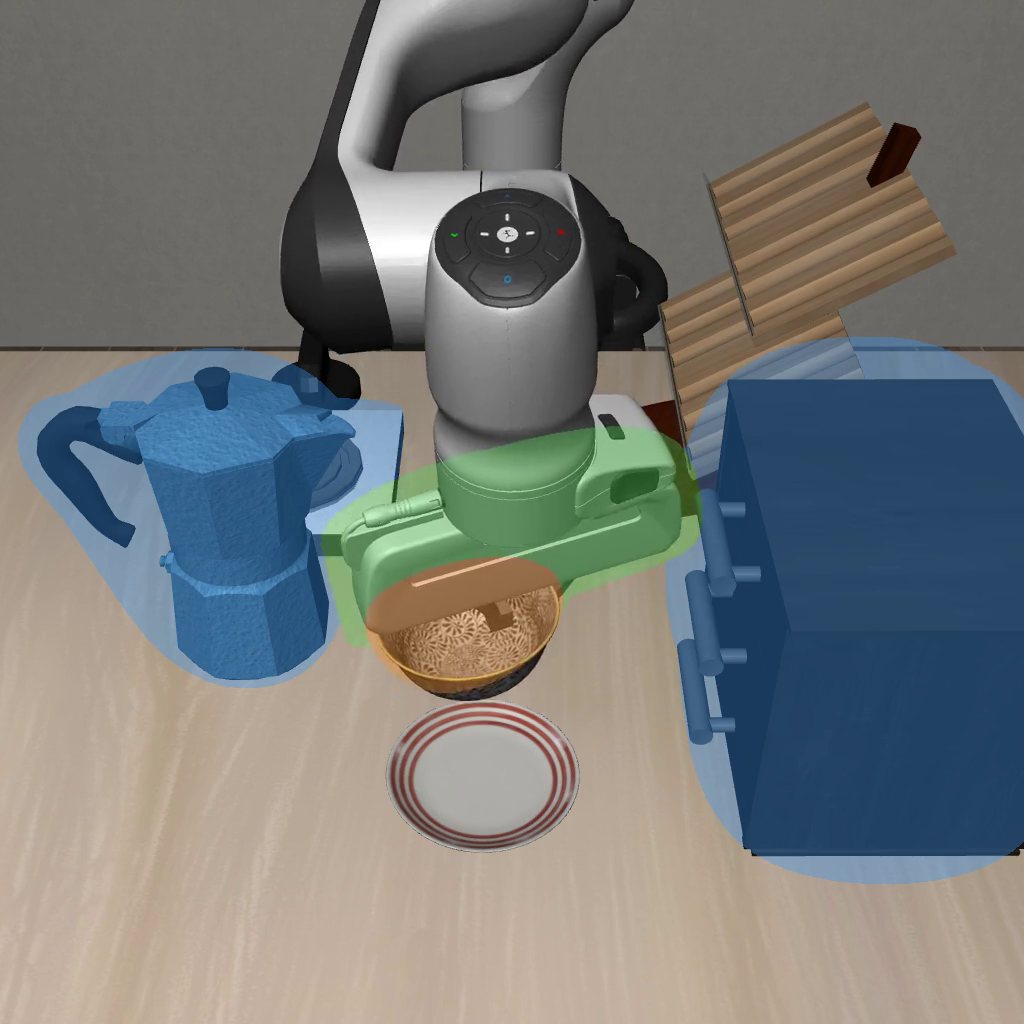}\\[-1pt]
        {\scriptsize\textbf{(c)} WBAG EEF-only\\[-1pt] + attached geometry}
    \end{minipage}\hfill
    \begin{minipage}[t]{0.20\textwidth}
        \centering
        \includegraphics[width=\linewidth]{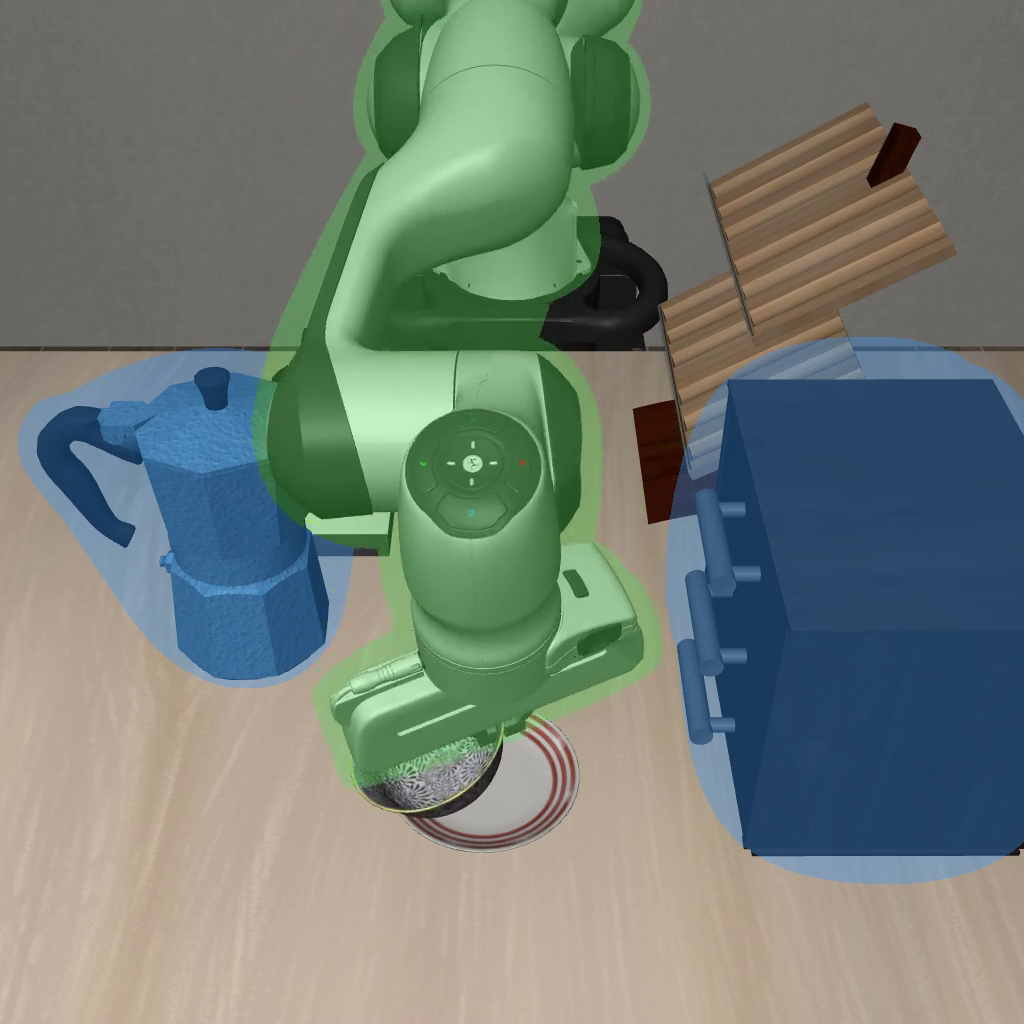}\\[-1pt]
        {\scriptsize\textbf{(d)} WBAG without\\[-1pt] attached geometry}
    \end{minipage}\hfill
    \begin{minipage}[t]{0.20\textwidth}
        \centering
        \includegraphics[width=\linewidth]{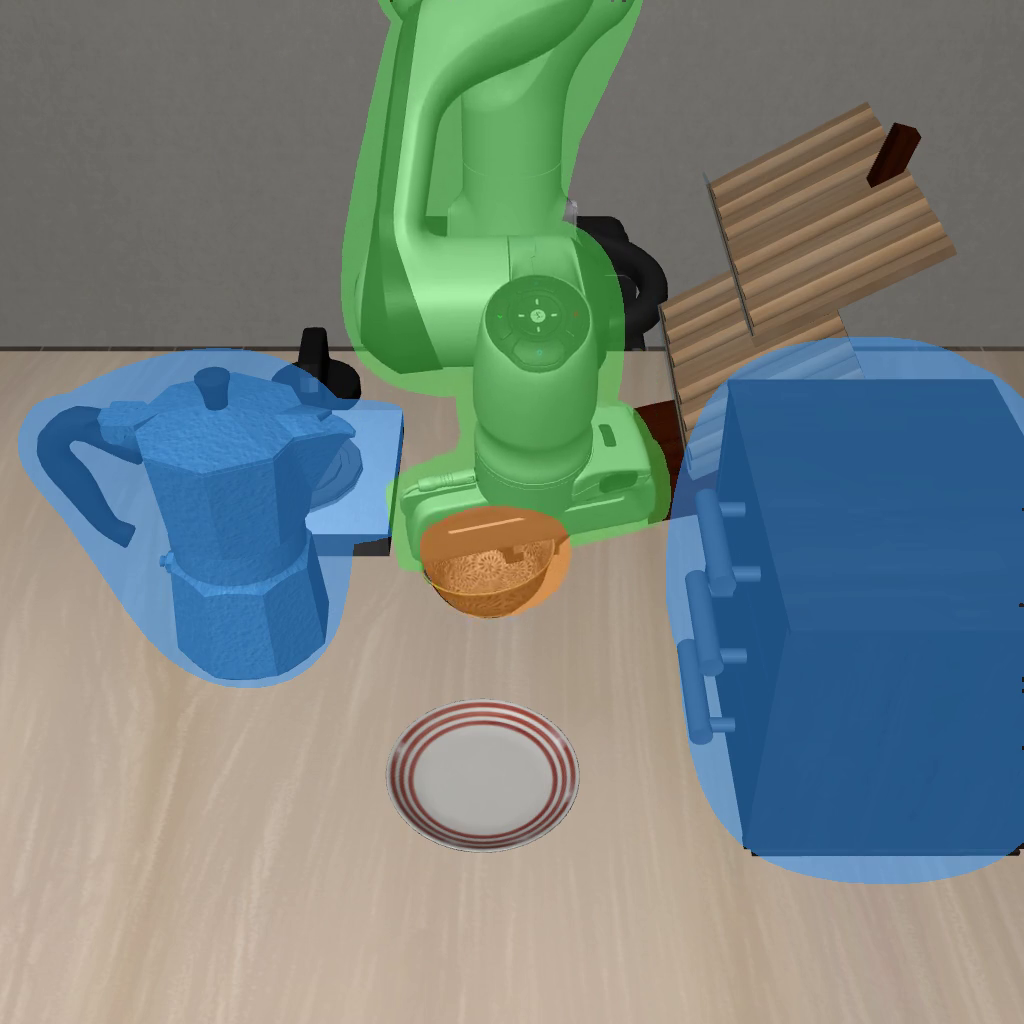}\\[-1pt]
        {\scriptsize\textbf{(e)} WBAG}
    \end{minipage}
    \caption{Visual comparison of the configurations in
    Table~\ref{tab:method_configurations}: (a) designated-obstacle EEF-MVEE,
    (b) scene-wide EEF-MVEE, (c) EEF-only with attached geometry, (d) whole-body
    without attached geometry, and (e) full WBAG.}
    \label{fig:method_geometry_comparison}
\end{figure*}

\subsection{Online Geometric Barrier Approximation}
Section~III defines $h_{bo}$ through a continuous closest-point search over $\Gamma_o$, which is impractical to solve for every protected-body--obstacle pair at each control step. WBAG therefore approximates this search using a finite set of local contact candidates. At control step $t$, let $\{x_{bo,k}^{t}\}_{k=1}^{N_{bo}^{t}}$ denote the candidate contact points for protected-body--obstacle pair $(b,o)$, where $N_{bo}^{t}$ is the number of candidates and $x_{bo,k}^{t}\in\mathbb{R}^{3}$ is the $k$th candidate point. For each candidate, we define
\begin{equation}
    \widehat h_{bo,k}^{\,t}(q)
    =
    \phi_b(x_{bo,k}^{t};q)-\ell_b,
    \label{eq:sampled_pairwise_barrier}
\end{equation}
where $\ell_b$ is the protected-body margin from~\eqref{eq:pairwise_barrier}. The candidate points are obtained by locally searching the obstacle level set $\Gamma_o$ for regions of small protected-body clearance. The deployed composite barrier for obstacle $o$ is then
\begin{equation}
    \widehat H_o^{\sigma,t}(q)
    =
    -\frac{1}{\beta}
    \log
    \left(
        \sum_{b\in\mathcal{B}_\sigma}
        \sum_{k=1}^{N_{bo}^{t}}
        \exp
        \left[
            -\beta
            \widehat h_{bo,k}^{\,t}(q)
        \right]
    \right),
    \label{eq:sampled_composite_barrier}
\end{equation}
using the same soft-min sharpness parameter $\beta$ as in~\eqref{eq:composite_barrier}.

\subsection{Action-to-Motion Mapping}
The barrier in Section~III is defined over robot configuration, whereas the frozen VLA outputs the six-dimensional operational-space increment $u$ in~\eqref{eq:vla_motion_action}. WBAG therefore maps this operational-space command to the local joint motion induced by the existing operational-space controller. Let $J(q)\in\mathbb{R}^{6\times n_q}$ denote the geometric Jacobian and define its damped pseudoinverse as
\begin{equation}
    J^\#_\lambda(q)
    =
    J(q)^\top
    \left(
        J(q)J(q)^\top+\lambda^2 I
    \right)^{-1},
    \label{eq:damped_pseudoinverse}
\end{equation}
where $\lambda>0$ is the damping coefficient. Let $D\in\mathbb{R}^{6\times6}$ contain the translational and rotational action scales of the downstream operational-space controller, and let $T_a>0$ denote the effective action horizon. We instantiate the local action-to-motion map $G(q)$ in~\eqref{eq:local_action_map} as
\begin{equation}
    \dot q
    =
    G_{\mathrm{OSC}}(q)u,
    \qquad
    G_{\mathrm{OSC}}(q)
    =
    J^\#_\lambda(q)\frac{D}{T_a}.
    \label{eq:osc_action_map}
\end{equation}
Thus, $G_{\mathrm{OSC}}(q)\in\mathbb{R}^{n_q\times6}$ maps the VLA's operational-space increment to an approximate local joint velocity. Applying this map to the deployed barrier from Section~IV-C gives
\begin{equation}
    \dot{\widehat H}_o^{\sigma,t}
    =
    \nabla_q \widehat H_o^{\sigma,t}(q_t)^\top
    G_{\mathrm{OSC}}(q_t)u.
    \label{eq:deployed_barrier_derivative}
\end{equation}
The modeled gripper geometry is treated as rigidly attached to the hand. For simplicity, finger motion induced by the scalar gripper command $g_{\mathrm{nom}}$ is not included in this local action-to-motion model.

\subsection{Grasp-Conditioned Geometric CBF-QP}
WBAG enforces the deployed grasp-conditioned CBF constraints through a minimum-intervention QP that selects the feasible command closest to the nominal VLA action:
\begin{equation}
\begin{aligned}
(u^\star,\delta^\star)
&=
\arg\min_{u,\delta\geq0}
\|W(u-u_{\mathrm{nom}})\|_2^2
+\lambda_\delta\delta^2\\
\text{s.t.}\qquad
&
\dot{\widehat H}_o^{\sigma,t}
+\alpha\!\left(
\widehat H_o^{\sigma,t}(q_t)
\right)
\geq -\delta,
\qquad \forall o\in\mathcal{O},\\
&
u_{\min}\leq u\leq u_{\max}.
\end{aligned}
\label{eq:wbag_qp}
\end{equation}
Here, $W\succeq0$ weights deviation from the nominal VLA command, $\delta\geq0$ is a shared slack variable, $\lambda_\delta>0$ penalizes CBF relaxation, and $u_{\min}$ and $u_{\max}$ are the native action bounds. For $\delta=0$, the constraint has the CBF form of~\eqref{eq:fixed_mode_cbf}, applied to the finite contact barrier $\widehat H_o^{\sigma,t}$. Because deployment uses this approximation and permits $\delta>0$, Proposition~\ref{prop:fixed_mode_safety} applies to the exact closest-point formulation with zero-slack, whereas the deployed filter instead uses finite contact search and permits slack to maintain QP feasibility. The resulting $u^\star$ is executed with the unchanged gripper command as in~\eqref{eq:safe_vla_action}.

\begin{table*}[t]
    \centering
    \caption{Performance on SafeLIBERO, aggregated over Levels I and II.}
    \label{tab:results_and_per_suite}
    \footnotesize
    \setlength{\tabcolsep}{2pt}
    \renewcommand{\arraystretch}{1.03}
    \begin{tabular*}{\textwidth}{@{\extracolsep{\fill}}lcccccccc@{}}
        \toprule
        Suite & Metric & \shortstack{Policy\\only} & \shortstack{AEGIS baseline} & \shortstack{SAM3\\EEF-MVEE} &
        \shortstack{SAM3 Scene\\EEF-MVEE} & \shortstack{WBAG EEF\\+ attached} &
        \shortstack{WBAG w/o\\attached} & \shortstack{\textbf{WBAG}} \\
        \midrule
        \multirow{3}{*}{Overall}
        & Scene Safety & 21.25\% & 70.87\% & 72.94\% & 82.75\% & 68.40\% & 77.50\% & \textbf{97.38\%} \\
        & Safe Success & 19.00\% & 51.06\% & 55.56\% & 46.25\% & 46.96\% & 46.31\% & \textbf{59.38\%} \\
        & Unsafe Success & 39.56\% & 11.25\% & 8.13\% & 4.31\% & 17.58\% & 14.69\% & \textbf{0.68\%} \\
        \midrule
        \multirow{3}{*}{Goal}
        & Scene Safety & 23.75\% & 83.50\% & 86.25\% & 77.25\% & 74.00\% & 88.75\% & \textbf{98.00\%} \\
        & Safe Success & 19.50\% & 72.75\% & \textbf{76.75\%} & 52.25\% & 47.50\% & 53.25\% & 58.00\% \\
        & Unsafe Success & 40.25\% & 7.00\% & 8.50\% & 8.25\% & 10.00\% & 6.75\% & \textbf{0.25\%} \\
        \midrule
        \multirow{3}{*}{Long}
        & Scene Safety & 22.25\% & 67.50\% & 64.75\% & 79.00\% & 62.00\% & 67.75\% & \textbf{96.25\%} \\
        & Safe Success & 20.00\% & 26.50\% & 30.50\% & 18.75\% & 34.00\% & 29.00\% & \textbf{48.75\%} \\
        & Unsafe Success & 36.25\% & 7.50\% & 7.75\% & 5.50\% & 21.00\% & 21.00\% & \textbf{2.25\%} \\
        \midrule
        \multirow{3}{*}{Object}
        & Scene Safety & 29.50\% & 59.50\% & 65.25\% & 90.00\% & 62.50\% & 71.00\% & \textbf{100.00\%} \\
        & Safe Success & 28.00\% & 47.50\% & 48.00\% & 46.25\% & 43.00\% & 45.00\% & \textbf{66.75\%} \\
        & Unsafe Success & 32.50\% & 21.50\% & 10.00\% & 1.25\% & 25.00\% & 23.25\% & \textbf{0.00\%} \\
        \midrule
        \multirow{3}{*}{Spatial}
        & Scene Safety & 17.75\% & 73.00\% & 75.50\% & 84.75\% & 69.25\% & 82.50\% & \textbf{95.25\%} \\
        & Safe Success & 17.00\% & 57.50\% & 67.00\% & \textbf{67.75\%} & 59.25\% & 58.00\% & 64.00\% \\
        & Unsafe Success & 44.50\% & 9.00\% & 6.25\% & 2.25\% & 12.75\% & 7.75\% & \textbf{0.25\%} \\
        \bottomrule
    \end{tabular*}
\end{table*}

\section{Results}
\label{sec:results}

\subsection{Evaluation Setup}
We evaluate on SafeLIBERO~\cite{vlsa_aegis}, a safety oriented extension of
LIBERO~\cite{libero} for tabletop manipulation. SafeLIBERO contains 16 tasks
drawn from the LIBERO Goal, Long, Object, and Spatial suites, with four tasks
from each suite. The Spatial tasks vary the spatial arrangement of the
manipulated objects, the Object tasks vary which object must be transported,
the Goal tasks vary the desired manipulation outcome, and the Long tasks
require longer sequential interactions involving multiple objects. Across
these suites, the robot performs pickup, transport, placement, articulated
interactions such as opening a drawer, and multi object manipulation.

SafeLIBERO augments each task with two safety obstacle placements. In Level I,
the designated obstacle is placed near the target object. In Level II, it is
placed farther from the target but along the robot's motion path. Each task
and safety level contains 50 randomized episodes, yielding 1,600 episodes in
total.

Following SafeLIBERO's displacement based safety proxy~\cite{vlsa_aegis}, we
evaluate all methods using a scene evaluator that monitors every eligible scene
body not required for task contact. An episode is unsafe if any such body is
displaced by more than $0.01$~m from its settled initial position $p_i(0)$:
\begin{equation}
    \left\|p_i(t)-p_i(0)\right\|_1 > 0.01~\mathrm{m}.
    \label{eq:scene_safety}
\end{equation}
Target objects, goals, and required support objects are excluded because
contact with them may be necessary for task completion.

We report three episode level metrics. \textbf{Scene Safety} holds when no
evaluated body violates Eq.~\eqref{eq:scene_safety}. \textbf{Safe Success}
denotes task completion with Scene Safety, whereas \textbf{Unsafe Success}
denotes task completion without Scene Safety.

All configurations use the same frozen $\pi_{0.5}$ checkpoint~\cite{pi05} and
identical episode initializations. Table~\ref{tab:method_configurations}
summarizes the evaluated configurations, and
Fig.~\ref{fig:method_geometry_comparison} visualizes their protected and
obstacle geometries. \emph{Policy Only} executes the frozen VLA without a
safety filter. \emph{AEGIS baseline} uses Grounding DINO
~\cite{grounding_dino} to localize the designated SafeLIBERO obstacle and
represents safety using an end effector minimum volume enclosing ellipsoid
(MVEE). \emph{SAM3 EEF MVEE} replaces Grounding DINO with SAM3~\cite{sam3}
while retaining the same designated obstacle and end effector representation.
\emph{SAM3 Scene EEF MVEE} extends obstacle coverage to all eligible scene
objects while retaining end effector protection. \emph{WBAG EEF only +
attached geometry} uses scene BP-SDFs and explicit attached object geometry
but protects only the end effector. \emph{WBAG without attached geometry}
uses whole body robot BP-SDFs but omits the grasped object from the protected
geometry. WBAG combines scene BP-SDFs, whole body robot protection,
and explicit attached object geometry.

To isolate the safety representation from semantic identification errors, all methods receive simulator provided identities for the target, goal,
and support objects when applicable. These identities replace the semantic
identification stage that would otherwise determine which objects require task
contact and should be excluded from the obstacle set. For \emph{AEGIS baseline} and \emph{SAM3 EEF MVEE}, simulator metadata
provides the identity of the designated SafeLIBERO safety obstacle. This
identity is used as the Grounding DINO~\cite{grounding_dino} query for \emph{AEGIS baseline}
and as the SAM3~\cite{sam3} prompt for \emph{SAM3 EEF MVEE}. Thus, both methods
localize the same designated obstacle from the visual observation and differ
only in the grounding model. \emph{SAM3 Scene EEF MVEE} and WBAG
configurations instead use the fixed SAM3 prompt set
\{\texttt{bottle}, \texttt{can}, \texttt{bowl}, \texttt{box},
\texttt{plate}, \texttt{container}, \texttt{black notebook}, \texttt{cup},
\texttt{blue moka pot}, \texttt{black handle}\} to localize eligible scene objects.

We instantiate the local action to motion map using controller scaling
$D=\operatorname{diag}(0.05,0.05,0.05,0.5,0.5,0.5)$ and effective action
horizon $T_a=0.35$~s. We use obstacle margin $\ell_o=0.01$~m and
protected body margin $\ell_b=0$. The soft minimum temperature, QP deviation
weight, and slack penalty are $\beta=100~\mathrm{m}^{-1}$,
$W=\operatorname{diag}(1,1,1,10,10,10)$, and $\lambda_\delta=10^4$,
respectively. WBAG protects all modeled moving robot links, including the
gripper and fingers, and adds the reconstructed target geometry after grasp
confirmation. The EEF MVEE configurations use the handcrafted EEF proxy
geometry from AEGIS~\cite{vlsa_aegis}, with fixed offset
$[0,0,-0.08]$~m and ellipsoid axes $[0.06,0.12,0.11]$~m by default, and
axes $[0.06,0.12,0.20]$~m for orange juice, milk, or alphabet soup.
Controller and evaluation settings are otherwise held fixed across
configurations, and BP-SDF fitting and contact search settings are fixed across
the WBAG configurations.

\subsection{Benchmark Performance}

Table~\ref{tab:results_and_per_suite} reports task and safety performance on
SafeLIBERO. Scene Safety increases from 70.87\% for \emph{AEGIS baseline} to
72.94\% for \emph{SAM3 EEF MVEE}, 82.75\% for
\emph{SAM3 Scene EEF MVEE}, and 97.38\% for full WBAG. Extending obstacle
coverage from \emph{SAM3 EEF MVEE} to \emph{SAM3 Scene EEF MVEE} improves
Scene Safety, but reduces Safe Success from 55.56\% to 46.25\%. WBAG
further increases Scene Safety to 97.38\% while achieving 59.38\% Safe Success
and reducing Unsafe Success to 0.68\%. These results show that broader scene
coverage alone improves safety at the cost of task completion, while WBAG improves safety without the same loss in Safe Success.

Across individual suites, WBAG achieves the highest Scene Safety and lowest
Unsafe Success in all four suites. WBAG also achieves the highest Safe Success
on Long and Object, while the EEF MVEE configurations achieve the highest Safe
Success on Goal and Spatial. This variation indicates that the effect of the
safety filter on task completion depends on the manipulation tasks within each
suite.

\begin{figure*}[!t]
    \centering
    \begin{minipage}[t]{0.31\textwidth}
        \centering
        \includegraphics[width=\linewidth]{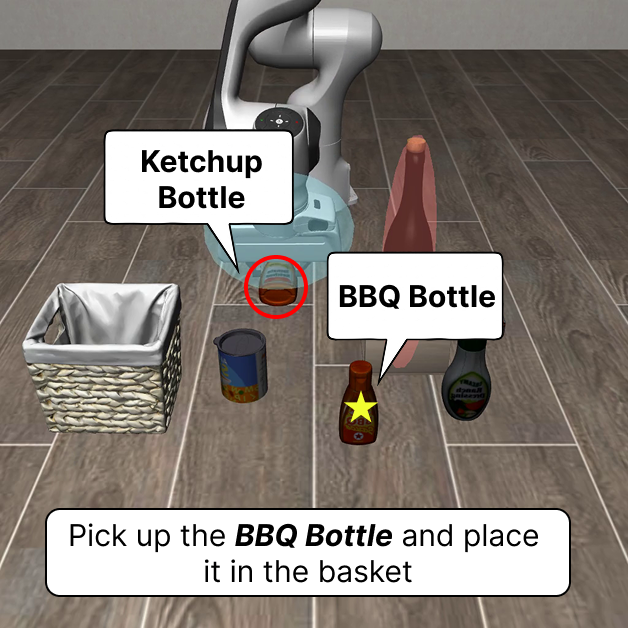}\\[-1pt]
        {\scriptsize\textbf{(a)} AEGIS baseline: incomplete scene coverage}
    \end{minipage}
    \begin{minipage}[t]{0.31\textwidth}
        \centering
        \includegraphics[width=\linewidth]{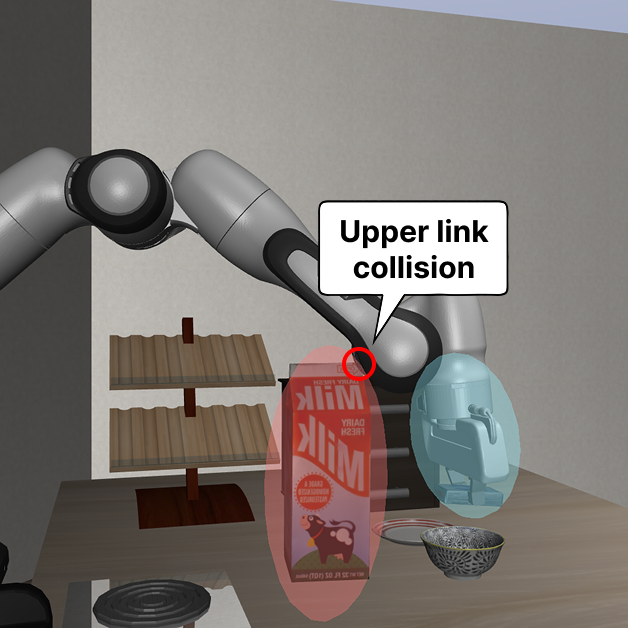}\\[-1pt]
        {\scriptsize\textbf{(b)} AEGIS baseline: upper-link collision}
    \end{minipage}
    \begin{minipage}[t]{0.31\textwidth}
        \centering
        \includegraphics[width=\linewidth]{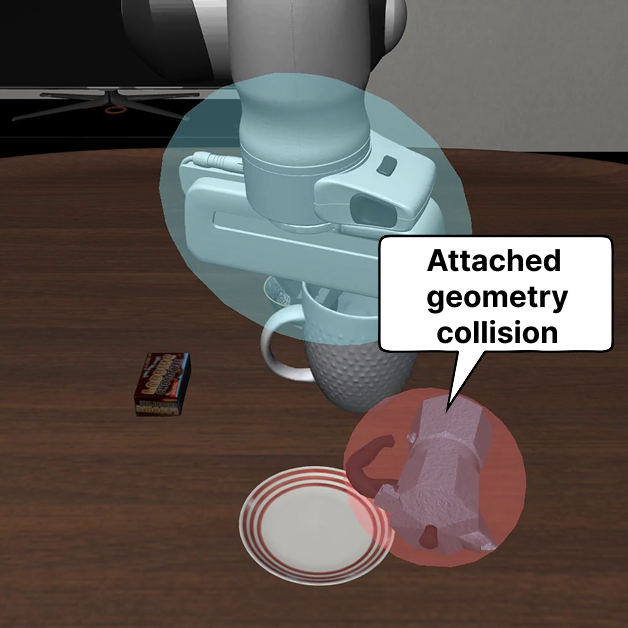}\\[-1pt]
        {\scriptsize\textbf{(c)} AEGIS baseline: attached-geometry collision}
    \end{minipage}
    \vspace{-3pt}
    \caption{Representative failure modes of simplified safety geometry:
    (a) unintended pickup of a non-task object excluded from the
    designated-obstacle constraint set, (b) upper-link collision under
    EEF-only protection, and (c) attached-object collision under a
    handcrafted EEF proxy.}
    \label{fig:failure_modes}
\end{figure*}

\subsection{Ablation and Qualitative Analysis}

The WBAG ablations in Table~\ref{tab:results_and_per_suite} isolate the effects
of whole body and attached object protection. Comparing WBAG with
\emph{WBAG EEF only + attached geometry} shows that whole body protection
improves both Scene Safety and Safe Success. Comparing WBAG with
\emph{WBAG without attached geometry} shows an additional benefit from
explicitly protecting the grasped object. Together, these comparisons show
that whole body and attached geometry modeling provide complementary benefits
for collision avoidance.

Figure~\ref{fig:failure_modes} illustrates the corresponding geometric failure
modes. Incomplete scene coverage can leave non task objects unconstrained.
End effector only protection can remain satisfied while an upstream robot link
collides with an obstacle. A handcrafted end effector proxy can also fail to
capture the shape and grasp relative pose of the manipulated object, allowing
collisions involving attached geometry.

\subsection{Runtime Analysis}
\begin{table}[!h]
    \vspace{2pt}
    \centering
    \caption{Mean per-step runtime for WBAG and policy inference.}
    \label{tab:runtime}
    \footnotesize
    \setlength{\tabcolsep}{3pt}
    \renewcommand{\arraystretch}{1.05}
    \begin{tabular*}{\columnwidth}{@{\extracolsep{\fill}}lc@{}}
        \toprule
        Component & Mean (ms) \\
        \midrule
        Contact search & 20.41 \\
        Robot BP-SDF eval. & 0.26 \\
        QP & 1.58 \\
        \midrule
        \textbf{Complete WBAG filter} & \textbf{26.76} \\
        Policy request & 15.43 \\
        \bottomrule
    \end{tabular*}
    \par\vspace{6pt}
    \noindent\parbox[t]{\columnwidth}{
    \scriptsize\textit{Per-step runtime is computed within each episode and
    then averaged across episodes. WBAG component timings are not exhaustive.}}
    \vspace{-18pt}
\end{table}

WBAG requires 26.76~ms per filtering step on average. Contact search accounts
for 20.41~ms, or approximately 76\% of the complete filter runtime, while
robot BP-SDF evaluation and QP optimization require only 0.26~ms and 1.58~ms,
respectively. The primary computational bottleneck is therefore the geometric
contact search used to approximate the closest point barriers rather than
barrier evaluation or QP optimization.

\section{Conclusion}
In this paper, 
we proposed WBAG, an inference-time safety framework for VLA manipulation that explicitly models whole-body and grasp-conditioned attached geometry. WBAG constructs BP-SDF clearance barriers over a grasp-conditioned protected set and enforces the resulting geometric CBF constraints directly on the VLA's native operational-space action, without retraining the VLA or modifying the operational-space controller.

On SafeLIBERO, WBAG achieves 97.38\% Scene Safety and 59.38\% Safe Success. Controlled ablations show complementary benefits from whole-body and attached-geometry protection, while the complete filter requires 26.76~ms per filtering step on average. Together, these results support explicitly modeling the full robot--attached-object system rather than relying on end-effector-centered safety geometry for runtime VLA manipulation.

The fixed-mode forward-invariance result applies to the ideal geometric barrier, whereas the deployed framework uses finite contact search, a local action-to-motion approximation, and QP slack. Our evaluation also relies on simulator-provided task-role identities and does not explicitly model uncertainty in reconstructed or attached geometry. Future work will investigate uncertainty-aware geometry and task-role inference, dynamic obstacles, policy-aware filtering that better preserves task completion, with validation in the real world.

\bibliographystyle{IEEEtran}
\bibliography{references}

\end{document}